\documentclass[twoside]{article}
\pdfoutput=1

\usepackage[accepted]{aistats2020}

\usepackage{color}
\usepackage{amsmath}
\usepackage{amsthm}
\usepackage{amsmath}
\usepackage{amssymb}
\usepackage{bbm}
\usepackage{verbatim}
\usepackage{textgreek}

\usepackage{cancel}
\usepackage[dvipsnames]{xcolor}
\usepackage{graphicx}
\usepackage{tgpagella}

\usepackage[colorinlistoftodos]{todonotes}
\usepackage{tikz}
\usepackage{bbm}
\usetikzlibrary{automata, arrows}
\usetikzlibrary{positioning}
\usepackage{mathtools}
\usepackage{mathrsfs}
\usepackage{enumerate}
\usepackage{multicol}
\usepackage{upgreek}
\usepackage{bigints}
\usepackage{physics}
\usepackage{float}
\restylefloat{table}
\usepackage{accents}
\usepackage{verbatim}
\usepackage[toc,page,title]{appendix}
\usepackage{bm}
\usepackage{enumitem}
\usepackage{booktabs}
\DeclareMathOperator{\Id}{I}

\DeclareMathOperator{\Borel}{Borel}

\DeclareMathOperator{\argmax}{argmax}

\newcommand{\Dists}{\mathscr{P}}
\newcommand{\tran}{\mathsf{T}}

\makeatletter
\newcommand{\subalign}[1]{%
  \vcenter{%
    \Let@ \restore@math@cr \default@tag
    \baselineskip\fontdimen10 \scriptfont\tw@
    \advance\baselineskip\fontdimen12 \scriptfont\tw@
    \lineskip\thr@@\fontdimen8 \scriptfont\thr@@
    \lineskiplimit\lineskip
    \ialign{\hfil$\m@th\scriptstyle##$&$\m@th\scriptstyle{}##$\hfil\crcr
      #1\crcr
    }%
  }%
}
\makeatother

\newcommand\numberthis{\addtocounter{equation}{1}\tag{\theequation}}
\renewcommand{\S}{\mathcal{S}}
\newcommand{\X}{\mathcal{X}}
\newcommand{\D}{\mathcal{D}}
\renewcommand{\P}{\mathcal{P}}
\newcommand{\A}{\mathcal{A}}
\newcommand{\M}{\mathcal{M}}
\newcommand{\N}{\mathcal{N}}
\newcommand{\Q}{\mathcal{Q}}
\newcommand{\R}{\mathcal{R}}
\newcommand{\U}{\mathcal{U}}
\newcommand{\T}{\mathcal{T}}
\newcommand{\C}{\mathcal{C}}
\newcommand{\K}{\mathcal{K}}
\newcommand{\F}{\mathcal{F}}
\newcommand{\B}{\mathcal{B}}
\newcommand{\E}{\mathcal{E}}
\newcommand{\W}{\mathcal{W}}
\newcommand{\V}{\mathcal{V}}
\newcommand{\G}{\mathcal{G}}

\renewcommand{\d}{\mathtt{d}}

\newcommand{\bE}{\mathbb{E}}

\newcommand{\bN}{\mathbb{N}}

\newcommand{\bP}{\mathbb{P}}
\newcommand{\bPc}[1]{\mathbb{P}\left\{#1\right\}}

\newcommand{\bR}{\mathbb{R}}

\newcommand{\g}[1]{\text{greedy}(#1)}

\renewcommand{\norm}[1]{\left\lVert#1\right\rVert}
\newcommand{\smallnorm}[1]{\lVert#1\rVert}
\renewcommand{\b}[1]{\bar{#1}}

\newtheorem{theorem}{Theorem}[section]
\newtheorem{lemma}{Lemma}[section]
\newtheorem{corollary}{Corollary}[theorem]
\newtheorem{prop}{Proposition}[section]

\theoremstyle{definition}
\newtheorem{definition}{Definition}[section]

\usepackage[backend=bibtex, citestyle=authoryear, bibstyle=authoryear, isbn=false, url=false, doi=false, eprint=false, natbib=true, sorting=nty, uniquename=init]{biblatex}

\begin{document}

\runningauthor{ Philip Amortila, Doina Precup, Prakash Panangaden, Marc G. Bellemare}

\twocolumn[

\aistatstitle{A Distributional Analysis of Sampling-Based Reinforcement Learning Algorithms}

\aistatsauthor{  Philip Amortila$^{\alpha}$ \And Doina Precup$^{\alpha,\beta}$ \And  Prakash Panangaden$^\alpha$ \And Marc G. Bellemare$^{\alpha,\beta,\gamma}$ }
\aistatsaddress{ $^\alpha$McGill University;\ \ $^\beta$CIFAR Fellow;\ \ $^\gamma$Google Research}]

\begin{abstract}
We present a distributional approach to theoretical analyses of reinforcement learning algorithms for constant step-sizes. We demonstrate its effectiveness by presenting simple and unified proofs of convergence for a variety of commonly-used methods. 
We show that value-based methods such as TD($\lambda$) and $Q$-Learning have update rules which are contractive in the space of distributions of functions, thus establishing 
their exponentially fast convergence to a stationary distribution. 
We demonstrate that the stationary distribution obtained by any algorithm whose target is an expected Bellman update has a mean which is equal to the true value function.
Furthermore, we establish that the distributions concentrate around their mean as the step-size shrinks.
We further analyse the optimistic policy iteration algorithm, for which the contraction property does not hold, and formulate a probabilistic policy improvement property which entails the convergence of the algorithm. 

\end{abstract}

\section{Introduction}
\label{sec:intro}
Basic results in the theory of Markov decision processes (MDPs) and dynamic programming (DP) rely on the two fundamental properties of the Bellman operator: contraction and monotonicity. For instance, proofs of convergence for value iteration and policy iteration follow immediately from the contractive properties of the Bellman operators and the Banach fixed point theorem \citep{szepesvari2010algorithms}.

However, proving the convergence of sample-based algorithms such as TD-learning \citep{sutton1988learning} or optimistic policy iteration \citep{tsitsiklis2002convergence} requires substantially more effort. The typical stochastic approximation approach relies on hitting-time or martingale arguments
to bound the sequence of value function iterates within progressively smaller regions (see, e.g., \cite[Section 4.3]{bertsekas1996neuro}).

In this work we present a distributional framework for analyzing sample-based reinforcement learning algorithms. Rather than consider the evolution of the random point estimate produced by the learning process, we study the dynamics of the \emph{distribution} of these point estimates. As a concrete example, we view the TD($0$) algorithm as defining a sequence of random iterates $(V_n)_{n \in \bN}$ whose distributions are recursively defined by the distributional equation
\begin{equation}\label{eq:td0}
V_{n+1}(s) \stackrel{D}{=} (1-\alpha) V_n(s) + \alpha \left(R(s,A) + \gamma V_n(S')\right) ,
\end{equation}
where $s$ is the initial state and $(A, R, S')$ is the random action-reward-next-state transition sampled from the underlying Markov Decision Process. The equation recursively describes the distribution of a random variable undergoing the TD learning process.

We study the constant step-size case. Our main contribution is to show that, for a variety of algorithms, the random iterates converge in distribution to a fixed point of the corresponding distributional equation, even though the random point estimate may not converge. We further characterize this fixed point, showing that it depends on both the step-size and the specific Markov Decision Process under consideration. 
Our framework views the learning process as defining a time-homogeneous Markov chain over the space of value functions. We prove convergence by establishing the existence of a stationary distribution for this Markov chain and demonstrating that the sequence of random iterates generated by a sample-based learning rule must converge to this stationary distribution, using tools from optimal transport \citep{Villani08}.

We first analyze sample-based algorithms whose corresponding distributional operator is a contraction mapping in the infinity norm, including TD$(\lambda)$, $Q$-learning, and double $Q$-learning. Following a proof technique of \citep{bach}, we lift these stochastic algorithms to the distributional setting. We show that this lifting recovers contractive guarantees, now in the Wasserstein metric using the infinity norm as a cost function. The contraction coefficient depends on the discount factor, as usual, but also on the step-size: updates with smaller step-sizes converge more slowly to their distributional fixed point. TD($0$), for example, is a contraction mapping with coefficient $1 - \alpha + \alpha \gamma$.

We also analyze the sample-based equivalent of policy iteration, called optimistic policy iteration  \parencite{tsitsiklis2002convergence} or Monte Carlo control \citep{sutton1998introduction}. The convergence of policy iteration is not driven by a contraction mapping, but rather by the monotonicity of the policy iteration operator \citep{puterman1994markov}. We derive a similar, weaker property for the sample-based setting which we call probabilistic policy improvement. We use this property to show that optimistic policy improvement also converges to a distributional fixed point.

By recovering the contraction mapping that underlies many dynamic programming algorithms, our distributional analysis significantly simplifies existing proofs of convergence for stochastic algorithms, at least for constant step-sizes. Our approach easily allows us to quantify the limiting behaviour of these algorithms; the same tool even provides us with confidence bounds over the true value function. We believe this type of analysis should prove useful going forward, including for the study of reinforcement learning with function approximation.

\section{Background}
\label{sec:background}

We write $\Dists(\X)$ for the set of probability distributions on a space $\X$. We consider an agent interacting with an environment modelled as a finite Markov decision process $\left(\S,\A,\R,\P,\gamma\right)$. As usual, $\S$ is a finite state space, $\A$ is a finite set of actions, $\R: \S \times \A \rightarrow \Dists([0,\textsc{Rmax}])$ is a bounded reward distribution function, $\P : \S \times \A \rightarrow \Dists(\S)$ is a transition distribution function, and $\gamma \in [0,1)$ is a discount factor. The strategy of the agent is captured by a policy $\pi : \S \rightarrow \Dists(\A)$. The value function $v^\pi : \S \rightarrow \bR$ of a policy $\pi$ is the expected discounted sum of rewards observed when starting at state $s$ and following policy $\pi$. The value function is the fixed point of the Bellman operator $\T^\pi$ defined by 
\begin{equation}\label{eq:bell-pi}
\T^\pi v(s) \coloneqq \bE_{\subalign{&a\sim\pi(\cdot|s)\\&r\sim\R(\cdot|s,a)}}\left[r + \gamma \bE_{s' \sim \P(\cdot|s,a)}\left[v(s')\right]\right].
\end{equation}
The value function of the optimal policy $\pi^\star$ is also the fixed point of the Bellman optimality operator $\T^\star$, defined by 
\begin{equation}\label{eq:bell-opt}
\T^\star v(s) \coloneqq \max_a\Big\{\bE_{\subalign{&r\sim\R(\cdot|s,a)\\&s'\sim\P(\cdot|s,a)}}\left[r + \gamma v(s')\right]\Big\}.
\end{equation}
A closely-related object is the \textit{action-value function} $q^\pi$, the expected discounted return of first taking action $a$ and thereafter following policy $\pi$. The action-value function satisfies the Bellman equations $ q^\pi(s,a) = \T^\pi q^\pi(s,a)$ and $q^\star(s,a) = \T^\star q^\star(s,a) $, where $\T^\pi$ and $\T^\star$ are defined analogously to Equations \eqref{eq:bell-pi} and \eqref{eq:bell-opt} \parencite{sutton1998introduction}.
The Bellman operators for value functions (resp. action-value functions) are contractions on $\bR^{|\S|}$ (resp. $\bR^{|\S|\times|\A|}$) with respect to the infinity norm $\norm{v} := \norm{v}_\infty = \max_i |v_i|$  \parencite{puterman1994markov}. A policy $\pi$ is called greedy with respect to an action-value function $Q \in \bR^{|\S|\times|\A|}$ if $\pi(s) \in \argmax_a{Q(s,a)}$ for each $s \in \S$.

\subsection{Couplings and the Wasserstein Metric}

To establish convergence in distribution, we will use the Wasserstein metric $\W$ between distributions \parencite{Villani08}. As a cost function, we use the infinity norm. For two distributions $\mu,\nu \in \Dists(\bR^\d)$, a pair of random vectors $(X,Y)$ is a coupling of $(\mu,\nu)$ if $X \sim \mu$ and $Y \sim \nu$. We write $\Xi(\mu,\nu)$ for the set of such couplings. The Wasserstein metric on $\Dists(\bR^\d)$ with the infinity norm as a cost function is defined as: 
\begin{equation}
\W(\mu,\nu) = \inf_{(X,Y) \in \Xi(\mu,\nu)} \bE\left[\norm{X-Y}_\infty\right]. \label{eq:wasserstein}
\end{equation}
The metric is defined over the set $\M(\bR^\d) = \left\{ \mu \in \Dists(\bR^\d): \int \norm{x}_\infty \mu(\dd{x}) < + \infty \right\}$ of measures with finite first moment. Assuming a bounded reward function, we will always be dealing with finite-moment measures. 
The Wasserstein metric characterizes the \textit{weak convergence} of measures \parencite[Theorem 6.9]{Villani08}, or equivalently the convergence in distribution of the associated random variables.

\section{Markov Chains on the Space of Functions}\label{sec:markov-processes}
With many value-based RL algorithms, the stochasticity of the algorithm depends only on the sampled transition and the random current estimate. For example, recalling the update rule for TD($0$) (Equation \eqref{eq:td0}), the value of $V_{n+1}(s)$ for a particular state is fully determined by knowledge of $V_n$ and the action, reward, and successor state which was sampled from $s$:
\begin{equation*}
    \bP\left\{V_{n+1} \mid V_n, V_{n-1},...,V_1,V_0\right\} = \bP\left\{V_{n+1} \mid V_n\right\}.
\end{equation*}
We therefore view these methods as inducing Markov chains on the space of value functions. We note that their state space is continuous rather than discrete -- we take it to be $\bR^{|\S|}$ when modelling value functions or $\bR^{|\S|\times|\A|}$ when modelling action-value functions. When results hold for both cases, we will write the discussion in terms of $\bR^{\d}$, $\d \in \bN$. Whenever needed, we may also restrict ourselves to the subset of realizable functions $[0, \frac{\textsc{Rmax}}{1 - \gamma}]^{\d} \subset \bR^{\d}$.

For a given update rule $\U$ and step-size $\alpha$, the transition function for an induced Markov chain is as follows. Given $f_k \in \bR^\d$, let $f_{k+1}$ be the random function obtained by applying $\U$ with step-size $\alpha$. For a Borel set $\B \in \texttt{Borel}(\bR^{\d})$, we define the Markov kernel $K_{\U,\alpha}$ as:
$$K_{\U,\alpha}(f_{k},\mathcal{B}) = \bP\left\{f_{k+1} \in \mathcal{B} | f_k \right\}.$$
 This Markov kernel describes the probability of transitioning from $f_{k}$ to some function in the set $\B$ under the update rule. In the sequel, we omit the subscripts on the kernel when the update rule is clear from context. 
 For a given probability measure $\mu \in \Dists(\bR^{\d})$, the distribution of functions after one transition of the Markov chain is given by $$\mu K (\B) = \int_{\bR^\d} K(\theta,\B)\mu(\dd{\theta}).$$

The distribution of functions after $n$ transitions is given by $K^n$, which is defined inductively as: $$K^{n}(\theta, \B) = \int_{\bR^{\d}} K(\theta',\B)K^{n-1}(\theta,\dd{\theta'}).$$ A probability measure $\psi$ is a \textit{stationary distribution} for a Markov chain with kernel $K$ if $$\psi= \psi K.$$
An algorithm updates synchronously when all states or state-actions pairs are updated at every iteration. In the regime of constant step-sizes and synchronous updates, the Markov kernels are \textit{time-homogeneous} (or time-independent). Thus, the law $\mu_n(\B) = \bPc{f_n \in \B}$ of the random variable $f_n$ is given by: $$\mu_{n} \leftarrow \mu_0 (K)^n.$$

\subsection{Stochastic operators}
In this section, we introduce the notion of a stochastic operator and provide a general formalism for the analysis of stochastic update rules. We will distinguish two classes of stochastic operators which will require different analyses.

We model the sampling space as a probability space $(\Omega, \F,\eta)$. A stochastic operator is a map between functions which depends on a randomly sampled event $\omega \in \Omega$. 
\begin{definition}[Stochastic operator]
A stochastic operator is a function $\widehat{\T} : \bR^\d \times \Omega \rightarrow \bR^\d$.
\end{definition}
When operating on functions, a stochastic operator $\widehat{\T}$ outputs a random function. We will write a number of stochastic value-based algorithms as
\begin{equation}\label{eq:general-alg}
f_{n+1} = (1-\alpha) f_n + \alpha \widehat{\T}(f_n,\omega ),
\end{equation}
where $f_n,f_{n+1} \in \bR^\d$ are functions, $\alpha$ is a step-size, and $\widehat{\T}$ is some algorithm-dependent stochastic operator. In this notation, the operator $\widehat{\T}$ is the \textit{target} of the algorithm. We say that $\widehat{\T}$ is an \textit{empirical Bellman operator} if it behaves like a Bellman operator in expectation.
\begin{definition}[Empirical Bellman Operator]
The stochastic operator $\widehat{\T}$ is an empirical Bellman operator for a policy $\pi$ if 
$$
\bE_{\omega \sim \eta}[\widehat{\T}(f,\omega)] = \T^\pi f \quad \forall f \in \bR^\d.
$$
Similarly,  $\widehat{\T}$ is an empirical Bellman \textit{optimality} operator if $\bE_{\omega \sim \eta}[\widehat{\T}(f,\omega)] = \T^\star f$.
\end{definition}

In general, the sampling distribution of the stochastic operator may depend on the function which it is acting on. 
Two examples of methods for which the sampling distribution is independent of the current function estimates are TD($0$), which applies an empirical Bellman operator, and $Q$-Learning, which applies an empirical Bellman optimality operator. TD($0$) is defined by the stochastic operator 
\begin{equation}\label{eq:td(0)_stochastic_op}
    \widehat{\T}(V,(a_s,r_s,s'_s)_{s \in \S})(s) = r_s + \gamma V(s'_s),
\end{equation} where $(a_s,r_s,s'_s)$ is a transition sampled for every state, and $Q$-Learning is defined by the operator
\begin{equation*}
\widehat{\T}(Q,(r_{s,a},s'_{s,a})_{(s,a)})(s,a) = r_{s,a} + \gamma \max_{a'}Q(s'_{s,a},a'),
\end{equation*}
where $(r_{s,a},s'_{s,a})$ is a transition sampled for every state-action pair. Convergence of these methods is covered in Section \ref{sec:contraction}. On the other hand, methods for which the sampling of the update rule depends on the function being updated are more akin to policy iteration, which applies Bellman operators that depend on the current greedy policy. We will see an example of such a method in Section \ref{sec:mcc}. 

\section{Convergence via Contraction to a Stationary Distribution}\label{sec:contraction}
In this section we demonstrate that common value-based algorithms converge to a stationary distribution when updated synchronously and with constant step-sizes. The convergence follows by showing that their Markov kernels are contractive with respect to the Wasserstein metric. To illustrate our approach, we provide a proof of convergence for TD($0$). With the same proof method, we also establish convergence and give convergence rates for Monte Carlo evaluation, $Q$-Learning, TD($\lambda$), SARSA, Expected SARSA \parencite{van2009theoretical}, and Double Q-Learning \parencite{hasselt2010double}. The proofs for these other algorithms are given in Appendix \ref{sec:laundry}.

Recall the update rule of the synchronous TD($0$) algorithm given by Equation \eqref{eq:td0}.
\begin{prop}\label{thm:contract}
For any step size $0<\alpha\leq 1$, the TD($0$) algorithm has a contractive Markov kernel $K_\alpha$:
\begin{equation}\label{eq:contract}
\W(\mu K_\alpha, \nu K_\alpha) \leq (1-\alpha+\alpha\gamma)\W(\mu,\nu),
\end{equation}
for all $\mu,\nu \in \M(\bR^{|\S|})$.
\end{prop}
\begin{proof}
Let $\mu^{(1)}, \mu^{(2)} \in \M(\bR^{|\S|})$ be two distributions of function estimates. Let $V^{(1)}_0 \sim \mu^{(1)}, V^{(2)}_0 \sim \mu^{(2)}$ be the coupling which minimizes the Wasserstein metric, i.e.:
 $$\W(\mu^{(1)},\mu^{(2)}) = \inf_{(X,Y)} \bE[ \norm{X-Y} ] = \bE\left[\smallnorm{V^{(1)}_0-V^{(2)}_0}\right].$$ 
 Such an optimal coupling always exists \parencite[Theorem 4.1]{Villani08}. We couple the updates $(V^{(1)}_1,V^{(2)}_1)$ to sample identical transitions at each state, i.e.:
\begin{equation}\label{eq:td-coupling}
\begin{aligned}
V^{(1)}_{1}(s) &= (1-\alpha)V^{(1)}_0(s) + \alpha\left(r_s + \gamma V^{(1)}_{0}(s'_s)\right) \\
V^{(2)}_{1}(s) &= (1-\alpha)V^{(2)}_0(s) + \alpha\left(r_s + \gamma V^{(2)}_{0}(s'_s)\right), 
\end{aligned}
\end{equation}
for the \textit{same} $a\  \sim \pi(\cdot|s),$ $r_s\  \sim \R(\cdot|s,a),$ and $s'_s \sim \P(\cdot|s,a)$. Note that this is a valid coupling of $(\mu^{(1)}K_\alpha,\mu^{(2)}K_\alpha)$ since $V_1^{(1)}$ and $V_1^{(2)}$ sample transitions from the same distributions. 
We upper-bound $\W( \mu^{(1)}K_\alpha,  \mu^{(2)}K_\alpha)$ with the coupling above. 
\begin{align}\label{eq:to-expand}
\W(\mu^{(1)} K_\alpha,\mu^{(2)} K_\alpha) &\leq \mathbb{E}\left[\smallnorm{V^{(1)}_1-V^{(2)}_1}\right] \nonumber \\ 
 &\leq (1-\alpha)\bE\left[ \smallnorm{V^{(1)}_0-V^{(2)}_0}\right] \nonumber \\ 
&\hspace{-3.2cm}+\alpha\bE\left[\max_s \big\lvert \left(r_s-r_s\right) + \gamma \big( V^{(1)}_0(s'_{s}) - V^{(2)}_0(s'_{s})\big) \big\rvert \right] 
\end{align}
We note that the expectation is over the pair $(V^{(1)}_0,V^{(2)}_0)$ as well as the random samples $a_s,r_s,s'_s$. By our coupling construction,
\begin{align*} \label{eq:emp-contraction}
&\hspace{-0.05cm}\bE\left[\max_s \big\lvert \left(r_s-r_s\right) + \gamma \big( V^{(1)}_0(s'_{s}) - V^{(2)}_0(s'_{s})\big) \big\rvert \right] \\
&\hspace{-0.05cm}=\gamma \bE \left[ \max_s \lvert V^{(1)}_0(s'_s) - V^{(2)}_0(s'_{s}) \rvert \right] \\
&\hspace{-0.05cm}\leq \gamma \bE \left[ \max_s  \lvert V^{(1)}_0(s) - V^{(2)}_0(s) \rvert \right] = \gamma\W(\mu^{(1)},\mu^{(2)}) \numberthis
\end{align*}
The inequality follows since $V_0^{(1)}$ and $V_0^{(2)}$ sample the same set of successor states -- the maximum is the same if each $s$ samples a different $s'_s$ and is lesser otherwise. Using Equation \eqref{eq:emp-contraction} in Equation \eqref{eq:to-expand} gives:
\begin{equation*}
\W(\mu^{(1)} K_\alpha,\mu^{(2)} K_\alpha) \leq (1-\alpha + \alpha\gamma)\W(\mu^{(1)},\mu^{(2)}). 
\end{equation*} 
Since $1-\alpha + \alpha\gamma< 1$, the kernel $K_\alpha$  is a contraction mapping.
\end{proof}

The contraction property readily entails the convergence to a stationary distribution. We initialize with any $V_0$ drawn from an arbitrary distribution of finite first moment. 
\begin{theorem}\label{thm:stat}
For any constant step size $0<\alpha \leq 1$ and initialization $V_0 \sim \mu_0 \in \M(\bR^{|\S|})$, the sequence of random variables $(V_n)_{n \geq 0}$ defined by the recursion (\ref{eq:td0}) converges in the Wasserstein metric to a unique stationary distribution $\psi^\text{TD(0)}_\alpha \in \M(\bR^{|\S|})$.
\end{theorem}
\begin{proof}
 The space of probability measures $\M(\bR^{|\S|})$ metrized with $\W$ is a complete metric space (\cite[Theorem 6.16]{Villani08}), and therefore it follows from Banach's fixed point theorem  that $(\mu_0 K_\alpha^n)_{n \geq 0}$ converges exponentially quickly to a unique fixed point $\psi^\text{TD(0)}_\alpha$. The distribution $\psi^\text{TD(0)}_\alpha$ is a stationary distribution by the fixed point property: $$\psi^\text{TD(0)}_\alpha K_\alpha = \psi^\text{TD(0)}_\alpha. \qedhere$$ 
\end{proof}

As evidenced by the above, lifting the analysis to distributions over value functions greatly simplifies the proof. The key is in the choice of a proper coupling. The same technique extends to a broad class of algorithms, with relatively few modifications. This avoids, for example, the additional hurdles caused by the greedy probability kernel in Q-learning \citep{tsitsiklis1994asynchronous}. We further note some surprising connections with distributional reinforcement learning \citep{bellemare2017distributional}. For $\alpha = 1$, the fixed point of TD($0)$ is in fact \citeauthor{bellemare2017distributional}'s return distribution. The same coupling, which forces two processes to sample the same transitions, has also been implicitly used to study the behaviour of distributional algorithms \citep{lyle2019comparative}.

\begin{table*}[t]
\begin{center}
\addtolength{\tabcolsep}{-0.8pt}
\small{
\begin{tabular}{@{}lccccccc@{}}
\toprule
& MC Evaluation & TD($\lambda)$ & SARSA & Expected SARSA & QL & Double QL \\ 
\midrule
Contraction factor & $ 1 - \alpha$ & $1 - \alpha + \alpha \gamma \frac{1 - \lambda}{1 - \lambda \gamma}$ & $1 - \alpha + \alpha \gamma$ & $1-\alpha+\alpha\gamma$ & $1 - \alpha + \alpha \gamma$ &  $\frac{1}{2}(2 - \alpha + \alpha \gamma$)\\
\bottomrule
\end{tabular}
}
\caption{Different sample-based algorithms which imply a contraction mapping in the Wasserstein metric over distributions on value functions. For each method, we also provide the corresponding contraction factor. Acronyms: Monte Carlo (MC), Q-Learning (QL). \label{table:laundry_list}} 
\end{center}
\end{table*}

To demonstrate the power of the approach, we summarize in Table \ref{table:laundry_list} a series of results regarding common sampling-based RL algorithms. Under similar conditions to Theorem \ref{thm:stat}, each algorithm listed in Table \ref{table:laundry_list} converges to a stationary distribution (which is in general different for different algorithms, as we show in the next section). Each proof only requires small adjustments to the basic proof template, for example an extended state space (Double Q-Learning). Full details, along with the proof template, are given in the appendix.

\section{The Stationary Distributions}\label{sec:stationary}
In this section, we characterize the stationary distributions which are attained by any algorithm whose target is a Bellman operator or Bellman optimality operator in expectation. In our notation, these algorithms are defined in terms of \textit{empirical} Bellman (optimality) operators. As before, we write the discussion in terms of $\bR^\d$ since results will hold for both value functions and action-value functions.

What do these distributions look like? We first consider the case of policy evaluation algorithms, which have as expected operator $\T^\pi$. In that case, their mean corresponds to the fixed point of $\T^\pi$, i.e. the value functions $v^\pi$ or $q^\pi$. Second, they concentrate around this mean in inverse proportion to the step-size $\alpha$. Hence, as expected, small step sizes lead to a more accurate distribution at the cost of a larger contraction factor. The full distributions are not symmetric or easily described, however; as a simple example, take $\alpha = 1$ in TD($0$), corresponding to the return distribution \citep{bellemare2017distributional}. In the case of optimality operators, we show that the mean of the stationary distributions is in fact greater than the fixed points $v^\star$ or $q^\star$. 

\subsection{Sample-based Evaluation Algorithms}

\begin{theorem}\label{thm:mean}
Suppose $\widehat{\T}^\pi$ is an empirical Bellman operator for some policy $\pi$ and that the updates \eqref{eq:general-alg} with step-size $\alpha$ converge to a stationary distribution $\psi_\alpha$. Let $f_\alpha \sim \psi_\alpha$ and $f^\pi$ be the fixed point of $\T^\pi$. %
Then $\bE[f_\alpha] = f^\pi$. 
\end{theorem}

\begin{proof}
Let $f_0$ be distributed according to $\psi_\alpha$. By stationarity,
\begin{equation}\label{eq:td-mean}
f_1  = (1-\alpha)f_0 + \alpha\widehat{\T}^\pi (f_0,\omega)
\end{equation}
is also distributed according to $\psi_\alpha$. We write $\overline{f_\alpha} \coloneqq \bE\left[f_0\right]$. Taking expectations on both sides,
and using that $\bE_\omega[\widehat{\T}^\pi(f,\omega))] = \T^\pi(f)$ for any $f$: 
\begin{align*}
\overline{f_\alpha}  &= (1-\alpha) \overline{f_\alpha}  + \alpha \bE_{\omega,\psi_\alpha}[\widehat{\T}^\pi(f_0,\omega)] \\
\overline{f_\alpha} &= \bE_{\psi_\alpha}[\T^\pi f_0] \\
\overline{f_\alpha} &= \T^\pi \bE_{\psi_\alpha}[f_0] = \T^\pi \overline{f_\alpha} 
\end{align*}
And therefore $\overline{f_\alpha} = f^\pi$ since it is the unique fixed point of  $\T^\pi$.
\end{proof}

We remark again that this characterization will hold for any algorithm which converges and performs Bellman updates in expectation. Although they have the same mean, the stationary distributions will depend on the update rule. These differences will be reflected in their higher moments. To this effect, we next derive a closed-form expression for the covariance of the stationary distribution. We write $A^\tran$ for the transpose of a matrix $A$. The outer product of two vectors $x,y \in \bR^{\d}$ is the matrix $xy^\mathsf{T} \in \bR^{\d \times \d}$ defined by $(xy^\mathsf{T})_{i,j} = x_iy_j$.
Thus, ${\bE[(\vec{X}-\vec{\mu})(\vec{X}-\vec{\mu})^\tran]}$ is the covariance of a random vector $\vec{X}$ with mean $\vec{\mu}$. The proof of the following result is provided in Appendix \ref{sec:stat-proofs}. 
\begin{theorem}\label{thm:cov}
Let $\widehat{\T}^\pi$ be an empirical Bellman operator for some policy $\pi$. Suppose $\widehat{\T}^\pi$ is such that the updates \eqref{eq:general-alg} with step-size $\alpha$ converge to a stationary distribution $\psi_\alpha$. Define $\xi_\omega(f) = \widehat{\T}^\pi(f,\omega) - \T^\pi f$, and  
$$\C(f) \coloneqq \bE_\omega[\xi_\omega(f)\xi_\omega(f)^\tran]$$
to be the covariance of the zero-mean noise term $\xi_\omega(f)$ for a given function $f$. Define $C = (1-(1-\alpha))^2$. The covariance of $f_\alpha \sim \psi_\alpha$ is given by
\begin{align*}
C\bE\left[(f_\alpha-f^\pi)(f_\alpha-f^\pi)^\tran\right]  &= \\ 
& \hspace{-9em} \alpha^2 (\gamma\P^\pi)\bE\left[(f_\alpha-f^\pi)(f_\alpha-f^\pi)^\tran\right](\gamma\P^\pi)^\tran \\
&\hspace{-9em} + \alpha(1-\alpha)(\gamma\P^\pi) \bE\left[(f_0-f^\pi)(f_0-f^\pi)^\tran\right] \\
&\hspace{-9em} + \alpha(1-\alpha) \bE\left[(f_\alpha-f^\pi)(f_\alpha-f^\pi)^\tran\right](\gamma\P^\pi)^\tran 
\\
&\hspace{-9em} + \alpha^2 \int \C(f)\psi_\alpha(\dd{f}). 
\end{align*}
\end{theorem}
Theorem \ref{thm:cov} provides a recursive definition for the covariance of the stationary distribution $\psi_\alpha$. The integral in the final line corresponds to the expected covariance of the empirical Bellman operator when sampling from the distribution. Under the assumption that states are updated independently, this ``one-step'' covariance is diagonal. More generally, the covariance matrix is scaled by $\alpha$, suggesting that the distribution concentrates around its mean when $\alpha$ is close to 0. The following makes this precise. We write $\norm{A}_\text{op}=\sup\left\{\norm{Av} : \norm{v}\leq1, v \in \bR^\d \right\}$ for the operator norm of a matrix $A$. 

\begin{corollary}
Assume that the state space of the Markov chain is bounded. Let $C \coloneqq (\frac{2\textsc{Rmax}}{1-\gamma})^2$. Then, we have that $\norm{\bE\left[(f_\alpha-f^\pi)(f_\alpha-f^\pi)^\tran\right]}_\text{op}$ is monotonically decreasing with respect to $\alpha$.  In particular, $ \lim_{\alpha \rightarrow 0} \norm{\bE[(f_\alpha-f^\pi)(f_\alpha-f^\pi)^\tran]}_\text{op} = 0, $ and we have that: 
\begin{align*}
 \bP\left\{ \min_i \lvert f_\alpha(i) - f^{\pi}(i)\rvert \geq \varepsilon \right\}  &\leq \frac{C}{\d\varepsilon^2} \frac{\alpha^2}{1-(1-\alpha+\alpha\gamma)^2} \\ &\stackrel{\alpha \rightarrow 0}{\longrightarrow} 0.
\end{align*}
\end{corollary}

We remark that the boundedness of the state space (e.g. by $[0,\frac{\textsc{Rmax}}{1-\gamma}]^\d \subset \bR^\d$) is easily satisfied in the presence of bounded rewards in the MDP. Furthermore, the above results about the mean and covariance can easily be extended beyond Bellman operators to any operator which has a unique fixed point and commutes with expectation. 

\subsection{Sample-Based Control Algorithms}

Above we saw that the mean of the stationary distribution of a sample-based method using a fixed policy is the value function for that policy. This no longer holds in the presence of optimality operators, for example in what is called the \emph{control setting} \citep{sutton1988learning}. To conclude this section, we use our distributional approach to highlight behavioural characteristics of control algorithms. 

\begin{theorem}\label{thm:q_learning}
Suppose $\widehat{\T}^\star$ is an empirical Bellman optimality operator such that the updates \eqref{eq:general-alg} with step-size $\alpha$ converge to a stationary distribution $\psi^\star_\alpha$. Let $f_\alpha \sim \psi^\star_\alpha$  and $f^\star$ is the fixed point of $\T^\star$. Then $$\bE[f_\alpha] \geq f^\star.$$  Equality holds if and only if the expectation and the maximum commute, i.e. $\bE \widehat{\T} f = \widehat{\T} \bE f$
\end{theorem}

\begin{proof}
As before, let $f_0$ be distributed according to $\psi^\star_\alpha$. Taking expectations on both sides of $f_1  = (1-\alpha)f_0 + \alpha\widehat{\T}^\star(f_0,\omega)$ and writing $\overline{f_\alpha} \coloneqq \bE\left[f_\alpha\right]$ gives:
\begin{align*}
\overline{f_\alpha}  &= (1-\alpha) \overline{f_\alpha}  + \alpha \bE_{\omega,f_0}[\widehat{\T}^\star(f_0,\omega)] \\
\overline{f_\alpha} &= \bE_{f_0}[\max_\pi \T^\pi f_0] \\
\overline{f_\alpha} &\geq \max_\pi \bE_{f_0}[\T^\pi f_0 ] \\
\overline{f_\alpha} &\geq \max_\pi \T^\pi \overline{f_\alpha} = \T^\star \overline{f_\alpha}
\end{align*}
By the linear programming formulation of MDPs \citep[Section 6.9.1]{puterman1994markov}, we conclude that $\bar{f_\alpha} \geq f^\star = \min_f \{f \geq \T^\star f \}$.
\end{proof}

The theorem shows that in general, sample-based control methods such as Q-learning produces a biased (in an expected sense) estimate of the optimal Q-value. This brings fresh evidence about the algorithm's well-known overestimation problem, which double Q-learning seeks to correct.

\section{Convergence via Monotonicity: Optimistic Policy Iteration}\label{sec:mcc}
In a previous section, we showed that a number of sampling-based algorithms induce a contraction mapping in the Wasserstein metric between distributions over value functions. In this section we analyze a non-contractive example, namely the optimistic policy iteration (OPI) algorithm. The OPI algorithm is a sampling-based analogue of the classic policy iteration (PI) algorithm. The latter is driven to convergence by the monotonicity of the greedy policy updates. We show in this section that our Markov chain approach can regain a distributional analogue of the monotonicity property, which we call \textit{probabilistic policy improvement}, and that this property can be used to analyze the algorithm in a restricted setting. 

The convergence of optimistic policy iteration is more difficult to prove than that of most sample-based algorithms, and has been previously been established for Robbins-Monro decreasing stepsizes by using monotonicity arguments and assumptions on the sampling distribution \citep{tsitsiklis2002convergence}. To the best of our knowledge, the convergence of OPI for more general conditions (including constant step-sizes) remained an open problem.

Optimistic policy iteration proceeds by constructing a greedy policy from its current value function, sampling one trajectory per state-action pair from this policy, then updating its value function towards the return of these trajectories. 
We will write 
\begin{equation*}
\G^\pi(s_0,a_0)=\sum_{t=0}^\infty \gamma^tr_t(s_t,a_t)
\end{equation*}
for a sampled discounted return starting at state $s_0$, taking first action $a_0$, and thereafter following policy $\pi$. For any $Q \in \bR^{|\S|\times|\A|}$, we write $\pi_Q$ for the greedy policy corresponding to $Q$ (assuming a consistent tie-breaking so that this is well-defined). Let $Q_0$ be some initial estimate and $\pi_0=\pi_{Q_0}$. The update rule of OPI is as follows ($\alpha \in (0, 1]$):
\begin{align}\label{eq:mcc}
Q_{n+1}(s,a) &= (1-\alpha)Q_n(s,a) + \alpha \G^{\pi_{n}}(s,a) \nonumber\\
\pi_{n+1} &= \pi_{Q_{n+1}}. %
\end{align}
Analyzing optimistic policy iteration in the distributional setting poses a few challenges. First, the distribution of sampled trajectories depends on the exact value function. Informally, the greedy mapping from value functions to policies induces a \emph{greedy partition} \citep[][, Figure 6.9]{bertsekas1996neuro}, with a different empirical Bellman operator corresponding to each region of this partition. This rules out a simple coupling argument, as functions with different greedy policies may have arbitrarily different return distributions. \citeauthor{bertsekas1996neuro} point out that optimistic policy iteration can lead to chattering, where the greedy policy fails to converge even the value function converges. For our analysis, we consider the simpler case $\alpha = 1$; we discuss the extension to $\alpha < 1$ at the end of the section.
\begin{theorem}\label{thm:alpha-1-convergence}
For $\alpha = 1$ and initialization $Q_0 \sim \mu_0 \in \M(\bR^{|\S|\times|\A|})$,
the sequence of random variables $(Q_n)_{n \geq 0}$ defined by the recursion (\ref{eq:mcc})
converges to a unique stationary distribution $\varphi_1 \in \Dists(\bR^{|\S|\times|\A|})$.
\end{theorem}
The key lemma is to extend the monotonicity property of the policy iteration operator to the distributional case. In policy iteration, the greedy policy $\pi'=\pi_{Q^\pi}$ with respect to $Q^\pi$ leads to an improved value function:
\begin{equation*}
    Q^{\pi'} \ge Q^{\pi} .
\end{equation*}
The role of $Q^{\pi}$ is therefore to provide us with the improved policy $\pi'$. We will show that the same holds true for the sampled returns: there is some probability that the greedy policy with respect to $\G^{\pi}$ is $\pi'$. This allows us to argue that there is a chance that optimistic policy iteration follows the correct ``greedy path'' to $\pi^*$. We will call this property \emph{probabilistic policy improvement}. 

We analyze the case $\alpha=1$ by considering a Markov chain over policies. Formally, $\Pi=\{\pi: \S \rightarrow \A \}$ will be our state space, with the Markov kernel:
\begin{align*}
K(\pi, \pi') &\coloneqq \bPc{\pi'= \pi_{\G^\pi}} \\ 
&= \bPc{ \pi' \texttt{ is greedy for } \G^\pi} .
\end{align*}
This Markov chain reflects the OPI process since, at every step, the greedy policy $\pi_n$ corresponding to $Q_n$ is sufficient to determine the distribution of $Q_{n+1}$. Since the set $\Pi$ of deterministic policies is finite, $K$ is a discrete Markov chain.
\begin{lemma}[Probabilistic policy improvement]\label{lemma:greedy-main}
Suppose 
$\pi' = \pi_{Q^\pi}$.
Then $K(\pi,\pi') > 0$.
\end{lemma}
The proof of Lemma \ref{lemma:greedy-main} is given in Appendix \ref{sec:mcc-proof}.
This shows that there is a nonzero probability that the chain improves on the current policy. 
This implies that there is some probability that OPI applied from $\pi^*$ produces $\pi^*$ as a greedy policy.
\begin{lemma}[$\pi^*$ is aperiodic]
The optimal policy $\pi^\star$ is aperiodic. In particular: $K(\pi^\star,\pi^\star)>0$. 
\end{lemma}
\begin{proof}
Since the optimal policy $\pi^\star$ is greedy with respect to $Q^\star$, from Lemma \ref{lemma:greedy-main} we conclude that $K(\pi^\star,\pi^\star) > 0$. 
\end{proof}
All that remains to show is that the optimal policy $\pi^\star$ is \textit{reachable} from any other policy with positive probability.
\begin{lemma}[$\pi^*$ is reachable from any initial $\pi_0$]\label{lemma:irreducible}
For every $\pi_0 \in \Pi$, there exists an $n(\pi_0) \in \bN$ such that $K^{n(\pi_0)}(\pi_0, \pi^\star) > 0$.
\end{lemma}
\begin{proof}
Let $\pi_0$ be an initial policy. Let $Q^{\pi_0}, Q^{\pi_1},...,Q^{\pi^\star}$ be the sequence of action-value functions obtained from classical PI. Since PI converges in a finite number of steps (say $n_{\pi_0}$) and is a deterministic process, this sequence is well-defined. For every $i \in \{1,..,n(\pi_0)\}$, we have that $K(\pi_i,\pi_{i+1}) > 0$ by Lemma \ref{lemma:greedy-main} (since $\pi_{i+1}$ is greedy with respect to $Q^{\pi_i}$ by construction). Thus we have that $K(\pi_0,\pi_1)K(\pi_1,\pi_2)\cdots K(\pi_{n(\pi_0)-1},\pi^\star) > 0$ and in particular $K^{n(\pi_0)}(\pi_0 , \pi^\star ) > 0$.
\end{proof}
Finally, the reachability and aperiodicity of $\pi^\star$ allow us to apply the ergodic theorem for finite Markov chains. 
\begin{proof}[Proof (of Theorem \ref{thm:alpha-1-convergence})]
The policy $\pi^\star$ must be contained in a communicating class $\mathscr{C}^\star$ of policies (perhaps consisting of only $\pi^\star$) which is aperiodic since $\pi^\star$ is. There may be other communicating classes in the Markov chain, but by Lemma \ref{lemma:irreducible} they must all be transient since they can reach $\pi^\star$. By the Markov chain convergence theorem \parencite[Theorem 4.3]{levin2017markov}, any initial distribution converges to a stationary distribution $\varphi_1 \in \Dists(\Pi)$ with support over $\mathscr{C}^\star$. 
\end{proof}
Our result shows that optimistic policy iteration, applied with a step-size of $\alpha = 1$, converges to a stationary distribution $\varphi_1$ over aperiodic policies (and thus to a stationary distribution over value functions through the possible returns of these policies).
Since $K(\pi^*, \pi^*) < 1$ in general, we know that this distribution has support on suboptimal policies; in fact, we know that
\begin{equation*}
    \varphi_1(\pi^*) = \frac{1}{1 - K(\pi^*, \pi^*)} \sum_{\pi \ne \pi^*} \varphi_1(\pi) K(\pi, \pi^*) .
\end{equation*}
By ``continuity'', this suggests that the algorithm should also converge for the general case $\alpha \in [0, 1)$. Unfortunately, our proof technique does not immediately carry over. The issue is that, for $\alpha < 1$, we no longer have a Markov chain over policies: the greedy policy depends on the history of past policies, through the value function. One path forward may be to study the Markov chain over value functions, but the known brittleness of optimistic policy iteration suggests that its distributional behaviour may be quite complex. In particular, the transition kernel fails to satisfy many basic properties (such as the weak Feller property) which are typically used to establish convergence in Markov chains over continuous spaces \citep{meyn2012markov}. We leave as open questions whether the algorithm does converge, and to which distribution.

\section{Related Work}\label{sec:related}
Some of our methods are inspired from the work of \citet{bach}, which develops the theory of constant step-size stochastic gradient descent (in the context of supervised learning). In particular, the proof method we present in Section \ref{sec:contraction} is inspired from the proof of their Proposition 2, although simplified and adapted to the RL setting, and the results in Section \ref{sec:stationary} follow the methods of their Proposition 3.

In RL, convergence in distribution results for constant step-sizes are typically derived using tools common to stochastic approximation theory such as the mean ODE method and Lyapunov functions (see, e.g., \citet[Chapter 8]{kushner2003stochastic} and \citet[Chapter 9]{borkar2009stochastic}). Examples of works which feature these methods include \citet{srikant2019finite,chen2019finite,lakshminarayanan2017linear,bhandari2018finite}. The results and methods of these works are different, as they neither exploit the Markov chain perspective nor establish the convergence of the iterates to a stationary distribution. 

Some works do make explicit use of the Markov chain perspective, most related are \citet{borkar2000ode,yu2016weak}. The first of these establishes the convergence of the Markov chains with respect to the Total Variation metric using tools from \parencite[Chapters 13-16]{meyn2012markov}. In applications to the analysis of RL algorithms, this type of convergence does not hold without restrictive assumptions such their Assumption (2.6) -- see Appendix \ref{sec:tv-conv} for a simple counterexample featuring a bandit with a single deterministic arm.
On the other hand, results about weak convergence of RL algorithms \parencite{yu2016weak} have established the convergence of the averaged iterates rather than the full sequence of distributions. The methods are also different, and rely on the weak Feller property \parencite{meyn2012markov} amongst other stochastic approximation techniques \parencite{kushner2003stochastic}.
As far as we are aware, the use of the Wasserstein metric and the result that RL algorithms are contractive with respect to this metric are novel.

\section{Conclusion and Future Work}
\label{sec:conclusion}
We studied the convergence properties of sample-based reinforcement learning algorithms by considering how they induce distributions over value functions.
Many of these algorithms are in fact contractive not in the space of functions but in the lifted space of distributions of functions. The proof methods relies on coupling the events sampled by two executions of the algorithm, and can be re-used for many algorithms. Using the same Markov chain approach, we also analyzed a restricted version of optimistic policy iteration, which is not amenable to a contraction mapping-type analysis. One of the key results is to make explicit that constant step-size reinforcement learning algorithms do converge, albeit in the weaker distributional sense. As an upside of using a constant step size, we obtain exponentially fast convergence (as indicated by the presence of a contraction factor). By controlling the step-sizes, the stationary distributions thus obtained can be tailored to yield values close to the true value function with high confidence. In the control setting, this should enable us to better explain the performance of practical reinforcement learning schemes.

Our work opens a number of interesting avenues for future research. First, it would be valuable to fully characterize the stationary distribution of sample-based methods, for example by deriving a closed-form expression for their characteristic functions. A deeper understanding of the distributions obtained by control algorithms is also of interest. Second, we did not analyze the case of decaying step-sizes or online updates, which would correspond to time-inhomogenenous Markov processes. More broadly, the coupling method has historically been invaluable for many applications in probability theory. It would be interesting to see if our approach can be applied to policy-based methods, for example policy gradient or actor critic, which are closer in spirit to optimistic policy iteration. Finally, the simplicity of our analysis suggests that it may be carried to the function approximation setting, perhaps eventually shedding light on the behaviour of reinforcement learning with nonlinear approximation methods such as deep networks.

\section*{Acknowledgements} 
We gratefully acknowledge funding from the CIFAR Learning in Machines and Brains program. We thank Adam Oberman for early discussions on this idea, and Nan Jiang for helpful conversations. We also thank the anonymous reviewers, Robert Dadashi, and Pablo Samuel Castro for feedback on earlier drafts.

\printbibliography

\onecolumn

\begin{appendices}

\section{Laundry List of Convergent Algorithms}\label{sec:laundry}

We outline the general proof recipe, which will be re-using for the following examples.

\paragraph{Proof strategy}\label{sec:proof-strat}
\begin{itemize}
\item[\textbf{(P1)}] Let $\mu^{(1)},\mu^{(2)}$ be initial distributions and $(f^{(1)}_0, f^{(2)}_0)$ be the optimal coupling which minimizes $\W(\mu^{(1)},\mu^{(2)})$;
\item[\textbf{(P2)}] Define an appropriate coupling $f^{(1)}_1 \sim \mu^{(1)} K, f^{(2)}_1 \sim \mu^{(2)} K$ -- e.g. by defining them to follow the same trajectories if the updates sample from the same distributions;
\item[\textbf{(P3)}] Use the upper bound $\W(\mu^{(1)} K ,\mu^{(2)} K) \leq \bE\left[\smallnorm{f^{(1)}_1-f^{(2)}_2}\right]$ and bound $\bE\left[\smallnorm{f^{(1)}_1-f^{(2)}_1}\right] \leq \rho \bE\left[\smallnorm{f^{(1)}_0-f^{(2)}_0}\right]$ for some $\rho$ which depends on $\gamma, \alpha,$ and other parameters of the algorithm. Pick the step-size $\alpha$ such that $\rho < 1$ to get that $\mu \mapsto \mu K$ is a contraction.
\end{itemize}

\subsection{Convergence of synchronous Monte Carlo Evaluation with constant step-sizes}

We prove that Monte Carlo Evaluation with synchronous updates \& constant step-size converges to a stationary distribution. The algorithm aims to evaluate the value function of a given policy $\pi$ using Monte Carlo returns. The update rule is given by:
\begin{equation}\label{eq:mce}
\forall \ s \in \S: \quad V_{n+1}(s) = (1-\alpha)V_n(s) + \alpha \G^\pi_n(s) \tag{MCE}
\end{equation}
where $\G^\pi_n(s) = \sum_{n \geq 0} \gamma^n r_n(s_n,a_n)$ is the return of a random trajectory $(s_n,a_n,r_n)_{n \geq 0}$ starting from $s$, following $a_n \sim \pi(\cdot|s_n), r_n \sim \R(\cdot|s_n,a_n)$, and $s_{n+1} \sim \P(\cdot|s_n,a_n)$. 
\begin{theorem}\label{thm:stat_mce}
For any constant step size $0<\alpha \leq 1$ and initialization $V_0 \sim \mu_0 \in \M(\bR^{|\S|})$, the sequence of random variables $(V_n)_{n \geq 0}$ defined by the recursion (\ref{eq:mce}) converges in distribution to a unique stationary distribution $\varphi_\alpha \in \M(\bR^{|\S|})$.
\end{theorem}
\begin{proof}
Following the proof strategy outlined above, we skip to step \textbf{(P2)} of the proof. We define the coupling of the updates
$(V^{(1)}_1,V^{(2)}_1)$ to sample the same trajectories:
\begin{equation}\label{eq:mce-coupling}
\left.\begin{aligned}
V^{(1)}_{1}(s) &= (1-\alpha)V^{(1)}_0(s) + \alpha \G^\pi_k(s) \\
V^{(2)}_{1}(s) &= (1-\alpha)V^{(2)}_0(s) + \alpha \G^\pi_k(s).
\end{aligned}\right\rbrace \text{for the \underline{same} $\G^\pi_k(s)$}
\end{equation}
Note that this is a valid coupling of $(\mu^{(1)}K_\alpha,\mu^{(2)}K_\alpha)$, since $V^{(1)}_{1}(s)$ and $V^{(2)}_{1}(s)$ have access to the same sampling distributions. 
We upper bound $\W( \mu^{(1)}K_\alpha,  \mu^{(2)}K_\alpha)$ by the coupling defined in Equation \eqref{eq:mce-coupling}. This gives:
\begin{align*}
\W(\mu^{(1)} K_\alpha,\mu^{(2)} K_\alpha) &\leq \mathbb{E}\left[\norm{V^{(1)}_1-V^{(2)}_1}\right]  \\
&= \mathbb{E}\left[\norm{(1-\alpha)V^{(1)}_0 + \alpha \G^\pi_1 - \left((1-\alpha)V^{(2)}_0 + \alpha \G^\pi_1 \right)}\right] \\
&= \bE\left[\norm{(1-\alpha)(V^{(1)}_0-V^{(2)}_0)}\right] \\
&= (1-\alpha)\bE\left[\norm{V^{(1)}_0-V^{(2)}_0}\right] = (1-\alpha) \W(\mu^{(1)},\mu^{(2)})
\end{align*}
Since $1-\alpha <1$, $K_\alpha$ is a contraction mapping and we are done.
\end{proof}

\subsection{Convergence of synchronous Q-Learning with constant step-sizes}
We prove that $Q$-Learning with synchronous updates \& constant step-sizes converges to a stationary distribution. The algorithm aims to learn the optimal action-value function $Q^\star$. The updates are given by:
\begin{equation}\label{eq:ql}
	\forall\  (s,a) \in \S \times \A: \quad  Q_{n+1}(s,a) = (1-\alpha)Q_n(s,a) + \alpha\left(r + \gamma \max_{a'} Q_n(s',a')\right), \tag{QL}
\end{equation}
where $r \sim \R(\cdot|s,a), s' \sim \P(\cdot|s,a)$, and $\alpha > 0$. 
\begin{theorem}\label{thm:stat_ql}
For any constant step size $0<\alpha \leq 1$ and initialization $Q_0 \sim \mu_0 \in \M(\bR^{|\S|\times|\A|})$, the sequence of random variables $(Q_n)_{n \geq 0}$ defined by the recursion (\ref{eq:ql}) converges in distribution to a unique stationary distribution $\xi_\alpha \in \M(\bR^{|\S|})$.
\end{theorem}

\begin{proof}
We use the proof outline given above, and jump straight to step \textbf{(P2)}. We witness the same-sampling coupling again:
\begin{equation*}\label{eq:ql-coupling}
\left.\begin{aligned}
Q^{(1)}_{1}(s,a) &= (1-\alpha)Q^{(1)}_0(s,a) + \alpha\left(r + \gamma \max_{a'}Q^{(1)}_0(s',a')\right) \\
Q^{(2)}_{1}(s.a) &= (1-\alpha)Q^{(2)}_0(s,a) + \alpha\left(r + \gamma \max_{a'}Q^{(2)}_{0}(s',a')\right)
\end{aligned}\right\rbrace \text{for the \underline{same} \parbox{6.3em}{ $r\  \sim \R(s,a),\\ s' \sim \P(\cdot|s,a)$}}
\end{equation*}
The bound follows similarly, but with one additional step. Again we write $\widehat{\T}(Q)(s,a) = r + \gamma \max_{a'} Q(s'_{(s,a)},a')$ for the empirical Bellman (optimality) operator. 
\begin{align*}
\bE\left[\norm{\widehat{\T}(Q^{(1)}) - \widehat{\T}(Q^{(2)})}\right] &=  \bE\left[\max_{s,a}\left\lvert r-r + \gamma\left(\max_{a'}Q^{(1)}(s'_{(s,a)},a') - \max_{a'}Q^{(2)}(s'_{(s,a)},a')\right)\right\rvert \right] \\
&= \gamma \bE\left[\max_{s,a}\left\lvert \max_{a'}Q^{(1)}(s'_{(s,a)},a') - \max_{a'}Q^{(2)}(s'_{(s,a)},a') \right\rvert \right]\\
&\leq \gamma \bE\left[\max_{s,a} \max_{a'} \left\lvert Q^{(1)}(s'_{(s,a)},a') - Q^{(2)}(s'_{(s,a)},a') \right\rvert \right]\\
&\leq \gamma \bE\left[ \max_{s,a}\left\lvert Q^{(1)}(s,a)-Q^{(2)}(s,a)\right\vert \right] = \gamma \bE\left[ \norm{Q^{(1)}-Q^{(2)}} \right] \qedhere
\end{align*}
The first inequality follows from $\lvert \max_{a'} Q_1(s,a') - \max_{a'}Q_2(s,a')\rvert \leq \max_{a'} \lvert Q_1(s,a')-Q_2(s,a')\rvert$, and the second inequality follows since $Q^{(1)}$ and $Q^{(2)}$ sampled the same $s'$. Concluding the proof as before we see that the kernel is contractive with Lipschitz constant $1+\alpha-\alpha\gamma < 1$, and we are done.
\end{proof}

\subsection{TD($\lambda$)}

We prove that TD$(\lambda)$ with synchronous updates \& constant step-size converges to a stationary distribution. The algorithm aims to evaluate the value function of a given policy $\pi$ using a convex combination of $n$-step returns. The update rule is given by:
\begin{equation}\label{eq:tdl}
\forall s: \  V_{n+1}(s) = (1-\alpha)V_n(s,a) + \alpha(1-\lambda)\sum_{k=1}^\infty \lambda^{k-1} \left(\sum_{i=0}^k \gamma^i r(s_i,a_i) + \gamma^k V_n(s_k)\right) \tag{TD($\lambda$)}
\end{equation}
where each $n$-step trajectory is sampled starting from $s$ and following policy $\pi$.
\begin{theorem}\label{thm:stat_tdl}
For any constant step size $0<\alpha \leq 1$ and initialization $V_0 \sim \mu_0 \in \M(\bR^{|\S|})$, the sequence of random variables $(V_n)_{n \geq 0}$ defined by the recursion (\ref{eq:tdl}) converges in distribution to a unique stationary distribution $\zeta_\alpha \in \M(\bR^{|\S|})$.
\end{theorem}
\begin{proof}
Again, we jump straight to step \textbf{(P2)} of the template given above. We couple every $n$-step trajectory to sample the same $n$ rewards, actions, and successors states.
\begin{equation*}\label{eq:tdl-coupling}
\left.\begin{aligned}
V^{(1)}_{k+1}(s) &= (1-\alpha)V^{(1)}_k(s) + \alpha(1-\lambda)\sum_{n=1}^\infty \lambda^{n-1}\left(\sum_{i=0}^{n-1} \gamma^i r_i(s_i,a_i) + \gamma^n V^{(1)}_{k}(s_n)\right) \\
V^{(2)}_{k+1}(s) &= (1-\alpha)V^{(2)}_k(s) + \alpha(1-\lambda)\sum_{n=1}^\infty \lambda^{n-1}\left(\sum_{i=0}^{n-1} \gamma^i r_i(s_i,a_i) + \gamma^n V^{(2)}_{k}(s_n)\right)
\end{aligned}\right\rbrace \parbox{14.5em}{ \underline{same} \\$(s_i,a_i,r_i)_{i=0}^n$ \\$\forall n$}
\end{equation*}

By the coupling, the reward terms will cancel in every n-step trajectory. We write $R^{(i)}_n=\sum_{i=0}^{n-1} \gamma^i r_i(s_i,a_i) + \gamma^n V^{(i)}_{k}(s_n)$ for the $n$-step return and $\hat{\T}(V)(s) = \sum_{k=1}^\infty \lambda^{k-1} \left(\sum_{i=0}^k \gamma^i r(s_i,a_i) + \gamma^k V_n(s_k)\right)$ for the empirical Bellman operator of TD($\lambda$).
\begin{align*}
\bE \left[ \norm{\hat{\T}(V^{(1)})-\hat{\T}(V^{(2)})} \right] &=  \bE\left[\max_s \left| \sum_{n=1}^\infty \lambda^{n-1}R^{(1)}_{n}- \sum_{n=1}^\infty \lambda^{n-1}R^{(2)}_{n} \right| \right] \\
&= \bE\left[\max_s \left| \sum_{n=1}^\infty \lambda^{n-1} \left(R^{(1)}_{n} - R^{(2)}_{n}\right) \right| \right] \\
&=  \bE\left[\max_s \left| \sum_{n=1}^\infty \lambda^{n-1}\gamma^n\left(V^{(1)}(s_n) - V^{(2)}(s_n)\right) \right| \right]  \tag{\text{reward terms cancel}} \\
&\leq \bE\left[ \sum_{n=1}^\infty \lambda^{n-1}\gamma^n \max_s \left| \left(V^{(1)}(s_n) - V^{(2)}(s_n)\right) \right| \right] \tag{triangle inequality} \\
&\leq \sum_{n=1}^\infty \lambda^{n-1}\gamma^n \bE\left[\max_s \left| V^{(1)}(s) - V^{(2)}(s) \right| \right] \tag{by the coupling} \\
&=\sum_{n=1}^\infty \lambda^{n-1}\gamma^n \bE\left[\norm{V^{(1)}-V^{(2)}}\right] = \gamma \frac{1}{1-\lambda\gamma} \bE\left[ \norm{V^{(1)}-V^{(2)}}\right]
\end{align*}  
Concluding the proof as before, we have $\W(\mu^{(1)}K,\mu^{(2)}K) \leq (1-\alpha + \alpha\gamma \frac{1-\lambda}{1-\lambda\gamma} )\W(\mu^{(1)},\mu^{(2)})$. Since $1-\alpha + \alpha\gamma \frac{1-\lambda}{1-\lambda\gamma}$ < 1 we are done. 
\end{proof}

\subsection{SARSA with $\varepsilon$-greedy policies}

In this example we will example the use of $\varepsilon$-greedy policies for control. In particular, we examine SARSA updates with $\varepsilon$-greedy policies. Let $\pi(\cdot|s)$ be some base policy. The updates are as follow:

\begin{equation*}\label{eq:sarsa}
	Q_{k+1}(s,a) = \begin{cases}
	(1-\alpha)Q_k(s,a) + \alpha\left(r(s,a) + \gamma Q_k(s',a')\right) &\text{ w.p. } \varepsilon \\
	(1-\alpha)Q_k(s,a)+\alpha\left(r(s,a)+\gamma \max_{a'} Q_k(s',a')\right) \quad &\text{ w.p. } 1-\varepsilon  \\
	\end{cases} \tag{SARSA} 
\end{equation*}
where $r \sim \R(\cdot|s,a)$ and $s' \sim \P(\cdot|s,a)$ in both cases and $a' \sim \pi(\cdot|s')$ in the first case.

\begin{theorem}
For any constant step size $0<\alpha \leq 1$ and initialization $Q_0 \sim \mu_0 \in \M(\bR^{|\S|\times |\A|})$, the sequence of random variables $(Q_n)_{n \geq 0}$ defined by the recursion (\ref{eq:sarsa}) converges in distribution to a unique stationary distribution $\theta_\alpha \in \M(\bR^{|\S|\times|\A|})$.
\end{theorem}
\begin{proof}
We jump straight to step \textbf{(P2)} of the proof template. We use the same-sampling coupling, where $Q_1^{(1)}$ takes the greedy action if and only if $Q_1^{(2)}$ does. In the non-greedy case, they sample the same $a' \sim \pi(\cdot|s')$. In all cases, both functions  sample the same $r(s,a)$ and $s'$.
We write $\hat{\T}(Q)(s,a) = \begin{cases} r+ \gamma Q(s',a') \text{ w.p. } \varepsilon \\
r+\gamma \max_{a'}Q(s',a') \text{ w.p. } 1-\varepsilon \end{cases}$ \\
The bound follows similarly to the examples of $Q$-learning and TD($0$). We omit the subscripts on the $Q$-functions.
\begin{align*}
\bE\left[\norm{\hat{\T}(Q^{(1)}) - \hat{\T}(Q^{(2)})}\right] &= \bP\left\{ \text{greedy action chosen} \right\}\bE\left[ \max_{s,a} \gamma \lvert (\max_{a'}Q^{(1)}(s',a') - \max_{a'}Q^{(2)}(s',a')  \rvert\right] \\ &\quad + \bP\left\{ \text{non-greedy action chosen}\right\}\bE\left[\max_{s,a} \lvert \gamma (Q^{(1)}(s',a')-Q^{(2)}(s',a'))\rvert \right] \\
&\leq \varepsilon\gamma \bE\left[\norm{Q^{(1)}-Q^{(2)}}\right] + (1-\varepsilon)\gamma \bE\left[\norm{Q^{(1)}-Q^{(2)}} \right] \\
&= \gamma \bE\left[\smallnorm{Q^{(1)}-Q^{(2)}}\right]
\end{align*}
The bound $\bE\left[ \max_{s,a} \gamma \lvert (\max_{a'}Q^{(1)}(s',a') - \max_{a'}Q^{(2)}(s',a')  \rvert\right] \leq \gamma \bE\left[\norm{Q^{(1)}-Q^{(2)}}\right]$ follows from $\lvert \max_{a'} Q_1(s,a') - \max_{a'}Q_2(s,a')\rvert \leq \max_{a'} \lvert Q_1(s,a')-Q_2(s,a')\rvert$, and since $Q^{(1)}$ and $Q^{(2)}$ sampled the same $s'$ in the greedy case. The bound $\bE\left[\max_{s,a} \lvert \gamma (Q^{(1)}(s',a')-Q^{(2)}(s',a'))\rvert \right] \leq  \bE\left[\norm{Q^{(1)}-Q^{(2)}} \right]$ follows since $Q^{(1)}$ and $Q^{(2)}$ sampled the same state-action pair in the non-greedy case. Concluding the proof as before, we have that $\bE\left[\smallnorm{Q_1^{(1)}-Q_1^{(2)}}\right]\leq (1-\alpha+\alpha\gamma) \bE\left[\smallnorm{Q_0^{(1)}-Q_0^{(2)}} \right]$, and thus the kernel is a contraction.
\end{proof}

\subsection{Expected SARSA with $\varepsilon$-greedy policies}

In this example we examine the Expected SARSA updates with $\varepsilon$-greedy policies. Let $\pi(\cdot|s)$ be some base policy. Define $\pi_\varepsilon(\cdot|s)$ as the $\varepsilon$-greedy policy which takes the greedy action with probability 1-$\varepsilon$ and $\pi$ otherwise. The updates are as follow:

\begin{equation*}\label{eq:exp-sarsa}
	Q_{k+1}(s,a) = (1-\alpha)Q_k(s,a) + \alpha\left(r(s,a) + \gamma \sum_{a'}\pi_\varepsilon(a'|s) Q_k(s',a')\right)
	\tag{Expected-SARSA} 
\end{equation*}
where $r \sim \R(\cdot|s,a)$ and $s' \sim \P(\cdot|s,a)$ in both cases and $a' \sim \pi(\cdot|s')$ in the first case.

\begin{theorem}
For any constant step size $0<\alpha \leq 1$ and initialization $Q_0 \sim \mu_0 \in \M(\bR^{|\S|\times |\A|})$, the sequence of random variables $(Q_n)_{n \geq 0}$ defined by the recursion (\ref{eq:exp-sarsa}) converges in distribution to a unique stationary distribution $\beta_\alpha \in \M(\bR^{|\S|\times|\A|})$.
\end{theorem}
\begin{proof}
We jump straight to step \textbf{(P2)} of the proof template. We use the same-sampling coupling.

We write $\hat{\T}(Q)(s,a) = r+ \gamma\sum_{a'}\pi(a'|s) Q(s',a')$.
The bound follows similarly to the examples of $Q$-learning and TD($0$). We omit the subscripts on the $Q$-functions.
\begin{align*}
\bE\left[\norm{\hat{\T}(Q^{(1)}) - \hat{\T}(Q^{(2)})}\right] &= \bE\left[ \max_{s,a} \gamma \lvert \sum_{a'}\pi_\varepsilon(a')Q^{(1)}(s',a') - \sum_{a'}\pi_\varepsilon(a')Q^{(2)}(s',a')  \rvert\right] \\
&\leq \bE\left[ \max_{s,a} \gamma \sum_{a'}\pi_\varepsilon(a') \vert Q^{(1)}(s',a') - Q^{(2)}(s',a') \rvert\right] \\
&\leq \bE\left[ \max_{s,a} \gamma \sum_{a'}\pi_\varepsilon(a') \norm{Q^{(1)}(s',a') - Q^{(2)}(s',a')} \right] \\
&\leq \gamma \bE\left[\smallnorm{Q^{(1)}-Q^{(2)}}\right]
\end{align*}
Concluding the proof as before, we have that $\bE\left[\smallnorm{Q_1^{(1)}-Q_1^{(2)}}\right]\leq (1-\alpha+\alpha\gamma) \bE\left[\smallnorm{Q_0^{(1)}-Q_0^{(2)}} \right]$, and thus the kernel is a contraction.
\end{proof}

\subsection{Double Q-Learning}\label{subsec:double_qlearning}

In this example we will have to modify our state-space and introduce a new metric on pairs of $Q$-functions.  The Double $Q$-Learning algorithm \parencite{hasselt2010double}\footnote{This is the original algorithm, not the deep reinforcement learning version given in \parencite{van2016deep}.} maintains two random estimates $(Q^A,Q^B)$ and updates  $Q^A$ with probability $p$ and $Q^B$ with probability $1-p$. Should $Q^A$ be chosen to be updated, the update is:
\begin{equation*}
Q^A_{n+1}(s,a) = (1-\alpha)Q^A_n(s,a) + \alpha\left(r(s,a) + \gamma Q^B_n (s,\argmax_{a'}Q^A_n(s',a'))\right).
\end{equation*}
Analogously, the update for $Q^B$ is:
\begin{equation*}
Q^B_{n+1}(s,a)=(1-\alpha)Q^B_n(s,a) + \alpha\left(r(s,a) + \gamma Q^A_n (s,\argmax_{a'}Q^B_n(s',a'))\right).
\end{equation*}
In both cases, we have $s' \sim \P(\cdot|s,a)$. For this algorithm, the updates are Markovian on \textit{pairs} of action-value functions. Thus we set the state space to be $\bR^{|\S|\times|\A|}\times\bR^{|\S|\times|\A|}$. We choose the product metric defined by $d_1((Q^{A},Q^{B}),(R^{A},R^{B})) = \norm{Q^{A}-R^{A}}+\norm{Q^{B}-R^{B}}$.

\begin{theorem}\label{thm:stat_dql}
For any constant step size $0<\alpha \leq 1$ and initialization $(Q^A_0,Q^B_0) \sim \mu_0 \in \M(\bR^{|\S|\times|\A|}\times\bR^{|\S|\times|\A|})$, the sequence of random variables $(Q^A_n,Q^B_n)_{n \geq 0}$ defined by the Double Q-Learning recursion converges in distribution to a unique stationary distribution $\chi_\alpha \in \M(\bR^{|\S|\times|\A|}\times\bR^{|\S|\times|\A|})$.
\end{theorem}
\begin{proof}
As before, let $\mu^{(1)},\mu^{(2)} \M(\bR^{|\S|\times|\A|}\times\bR^{|\S|\times|\A|})$ be arbitrary initializations and $(Q_0^{A},Q_0^{B})$ and $(R_0^{A},R_0^{B})$ be the optimal coupling of $\W(\mu^{(1)},\mu^{(2)})$. We couple $(Q_1^{A},Q_1^{B})$ and $(R_1^{A},R_1^{B})$ to sample the same function to be updated and the same $s'$. Assume for a moment that $Q^A$ and $R^A$ are chosen to be updated. Proceeding as in the proof of Q-Learning (cf. Theorem \ref{thm:stat_ql}), we find that
\begin{align*}
\bE\left[\norm{Q_1^{A}-R_1^{A}}\right] \leq (1-\alpha)\bE\left[\norm{Q_0^A-R_0^A}\right] + \alpha\gamma\bE\left[\norm{Q_0^B-R_0^B}\right].
\end{align*}
Analogously, if $Q^B$ and $R^B$ are chosen to updated, we have:
\begin{align*}
\bE\left[\norm{Q_1^{B}-R_1^{B}}\right] \leq (1-\alpha)\bE\left[\norm{Q_0^B-R_0^B}\right] + \alpha\gamma\bE\left[\norm{Q_0^A-R_0^A}\right].
\end{align*}
Putting everything together, the full expectation is:
\begin{align*}
\bE\left[d((Q_1^A,Q_1^B),(R_1^A,R_1^B))\right]&= \bE\left[\norm{Q_1^A-R_1^A}+\norm{Q_1^B-R_1^B} \right] \\
&=\bP\left\{\text{A is updated}\right\}\bE\left[ \norm{Q_1^A-R_1^A}+\norm{Q_1^B-R_1^B} \right] \\
&\quad +\bP\left\{\text{B is updated}\right\}\bE\left[ \norm{Q_1^A-R_1^A}+\norm{Q_1^B-R_1^B}\right] \\ 
&=p\bE\left[ \norm{Q_1^A-R_1^A}+\norm{Q_0^B-R_0^B} \right] \\
&\quad +(1-p)\bE\left[ \norm{Q_0^A-R_0^A}+\norm{Q_1^B-R_1^B}\right] \\ 
&\leq p\left((1-\alpha)\bE\left[\norm{Q_0^A-R_0^A}\right] + (1+\alpha\gamma)\bE\left[\norm{Q_0^B-R_0^B}\right]\right) \\
&\quad +(1-p)\left((1+\alpha\gamma)\bE\left[ \norm{Q_0^A-R_0^A}\right] + (1-\alpha)\bE\left[\norm{Q_0^B-R_0^B}\right] \right) \\ 
&\leq \frac{1}{2}(2+\alpha\gamma - \alpha)\left(\bE\left[ \norm{Q_0^A-R_0^A}\right] + \bE\left[\norm{Q_0^B-R_0^B}\right] \right)  \\ 
&= \frac{1}{2}(2+\alpha\gamma - \alpha) \bE\left[ d((Q_0^A,Q_0^B),(R_0^A,R_0^B)) \right]
\end{align*}
Since $0 \leq 1/2(2+\alpha\gamma - \alpha) < 1$, so we are done. We note that the first equality only follows since, under the coupling, either $A$ or $B$ is updated for both functions.
\end{proof}

\section{Proofs of Section \ref{sec:stationary}}\label{sec:stat-proofs}

\begin{theorem}\label{thm:mean}
Suppose $\widehat{\T}^\pi$ is such that the updates \eqref{eq:general-alg} with step-size $\alpha$ converge to a stationary distribution $\psi_\alpha$. If $\widehat{\T}$ is an empirical Bellman operator for some policy $\pi$, then $\bE[f_\alpha] = f^\pi$ where $f_\alpha \sim \psi_\alpha$ and $f^\pi$ is the fixed point of $\T^\pi$. 
\end{theorem}

\begin{proof}
Let $f_0$ be distributed according to $\psi_\alpha$. Rewriting equation (\ref{eq:general-alg}):
\begin{equation}\label{eq:td-mean}
f_1  = (1-\alpha)f_0 + \alpha\T^\pi f_0 + \alpha \xi (f_0),
\end{equation}
where $\xi(f_0) = \hat{\T}^\pi(f_0,\omega) - \T^\pi f_0$ is a zero-mean noise term. Taking expectations on both sides, and using that $f_1$ is also distributed according to $\psi_\alpha$ by stationarity and that $\bE[\xi(f)] = 0$ for any $f$: 
\begin{align*}
\overline{f_\alpha}  &= (1-\alpha) \overline{f_\alpha}  + \alpha \bE[\T^\pi f_0] \\
\alpha \overline{f_\alpha} &= \alpha \bE[\R^\pi + \gamma \P^\pi f_0] \\
\overline{f_\alpha} &= \R^\pi + \gamma \P^\pi \bE[f_0] \\
\overline{f_\alpha} &= \T^\pi \overline{f_\alpha} 
\end{align*}
And therefore $\overline{f_\alpha} = f^\pi$ since it is the unique fixed point of  $\T^\pi$.
\end{proof}

\begin{theorem}
Suppose $\widehat{\T}^\pi$ is such that the updates \eqref{eq:general-alg} with step-size $\alpha$ converge to a stationary distribution $\psi_\alpha$, and that $\widehat{\T}^\pi$ is an empirical Bellman operator for some policy $\pi$. Define $$\C(f) \coloneqq \bE_\omega[(\widehat{\T}^\pi(f,\omega)-\T^\pi f)(\widehat{\T}^\pi(f,\omega)-\T^\pi f)^\tran]$$ to be the covariance of the zero-mean noise term $\widehat{\T}^\pi(f,\omega)-\T^\pi f$ for a given function $f$.  Then, the covariance of $f_\alpha \sim \psi_\alpha$ is given by 
\begin{align*}
(1-(1-\alpha)^2)\bE\left[(f_\alpha-f^\pi)(f_\alpha-f^\pi)^\tran\right] &= \alpha^2 (\gamma\P^\pi)\bE\left[(f_\alpha-f^\pi)(f_\alpha-f^\pi)^\tran\right](\gamma\P^\pi)^\tran \\
&\quad + \alpha(1-\alpha)(\gamma\P^\pi) \bE\left[(f_0-f^\pi)(f_0-f^\pi)^\tran\right] \\
&\quad + \alpha(1-\alpha) \bE\left[(f_\alpha-f^\pi)(f_\alpha-f^\pi)^\tran\right](\gamma\P^\pi)^\tran \\
&\quad + \alpha^2 \int \C(f)\psi_\alpha(\dd{f})
\end{align*}

Furthermore, we have that $\norm{\bE\left[(f_\alpha-f^\pi)(f_\alpha-f^\pi)^\tran\right]}_\text{op}$ is monotonically decreasing with respect to $\alpha$, where $\norm{\cdot}_\text{op}$ denotes the operator norm of a matrix. In particular, $ \lim_{\alpha \rightarrow 0} \norm{\bE[(f_\alpha-f^\pi)(f_\alpha-f^\pi)^\tran]}_\text{op} = 0, $ and we have that:  
$$
\bP\left\{ \min_i \lvert f_\alpha(i) - f^{\pi}(i)\rvert \geq \varepsilon \right\} \stackrel{\alpha \rightarrow 0}{\longrightarrow} 0 \quad \forall \ \varepsilon>0
$$
\end{theorem}
We preface the proof with some useful identities. We will write the covariance in terms of the tensor product for ease of manipulations
\begin{lemma}\label{lem:cov-lemma}
Write $\xi(f) \coloneqq (\widehat{\T}^\pi(f,\omega)-\T^\pi f)$. In the same setup as Theorem \ref{thm:cov}: 
$$\bE\left[(f_\alpha-f^\pi)(\T^\pi f_\alpha - f^\pi + \xi(f_0))^\tran\right] = \bE\left[(f_\alpha-f^\pi)(f_\alpha-f^\pi)^\tran \right](\gamma \P^\pi)^\tran$$
and 
\begin{align*} \bE\left[\left(\left(\T^\pi f_\alpha-f^\pi\right)+\xi(f_\alpha)\right)\left(\left(\T^\pi f_\alpha-f^\pi\right)+\xi(f_\alpha)\right)^\tran\right]  &=  (\gamma \P^\pi)\bE\left[(f_\alpha-f^\pi)(f_\alpha-f^\pi)^\tran \right](\gamma \P^\pi)^\tran \\&\quad + \int C(v)\psi_\alpha(\dd{v}) 
\end{align*}
\end{lemma}
\begin{proof}
Let $f_0 \sim \psi_\alpha$, by \eqref{eq:general-alg} we have $f_1 = (1-\alpha)f_0 + \alpha(\T^\pi f_0 + \xi(f_0))$ and $f_1 \sim \psi_\alpha$. Furthermore, the distribution of $f_0$  is independent of the distribution of $\omega$. By independence,  
\begin{align*}
\bE\left[(f_0-f^\pi)\xi(f_0)^\tran\right] &= \bE_{f_0} \bE_{\omega} \left[(f_0-f^\pi) \xi(f_0)^\tran\right] \tag{by independence of $f_0$ and $\xi(\cdot)$} \\ 
&= \bE_{f_0}\left[(f_0-f^\pi) (\bE_{\omega}\xi(f_0))^\tran\right] = 0 \tag{$\bE_{\omega}[\xi(f)] = 0 $ for every $f$}
\end{align*} 
For the first identity, note that 
\begin{align*}
\bE\left[(f_0-f^\pi)(\T^\pi f_0-f^\pi))^\tran\right] &= \bE\left[(f_0-f^\pi)(\R^\pi + \gamma \P^\pi(f_0)- \R^\pi - \gamma \P^\pi(f^\pi))^\tran \right] \\
&= \bE\left[(f_0-f^\pi)(\gamma\P^\pi(f_0-f^\pi))^\tran\right] \\
&= \bE\left[(f_0-f^\pi)(f_0-f^\pi)^\tran(\gamma\P^\pi)^\tran\right] \\
&= \bE\left[(f_0-f^\pi)(f_0-f^\pi)^\tran\right](\gamma \P^\pi)^\tran \\
\end{align*}
The first identity then follows by using $\bE\left[(f_0-f^\pi)\xi(f_0)^\tran\right] = 0$ and linearity of expectations.

\noindent For the second identity, expanding the outer product gives:
\begin{align*}
\bE\left[\left(\left(\T^\pi f_0-f^\pi\right)+\xi(f_0)\right)\left(\left(\T^\pi f_0-f^\pi\right)+\xi(f_0)\right)^\tran\right] &= \bE\left[(\T^\pi f_0-f^\pi)(\T^\pi f_0-f^\pi)^\tran\right] \\ &\quad + \bE\left[(\xi(f_0))(\xi(f_0)))^\tran\right] \\ &\quad  + \cancel{\bE\left[(\T^\pi f_0-f^\pi)(\xi(f_0))^\tran\right]} \\ &\quad + \cancel{\bE\left[\xi(f_0)(\T^\pi f_0-f^\pi)^\tran \right]} \\
&= \bE\left[(\gamma\P^\pi(f_0-f^\pi))(\gamma\P^\pi(f_0-f^\pi))^{\tran}\right]  \\ &\quad + \int \C(v)\psi_\alpha(\dd{v}) \\
&= (\gamma P^\pi)\bE\left[(f_0-f^\pi)(f_0-f^\pi)^\tran \right](\gamma P^\pi)^\tran \\ &\quad + \int \C(v)\psi_\alpha(\dd{v})
\end{align*}
where we used $\bE\left[(\T^\pi f_0-f^\pi)(\xi(f_0))^\tran\right] = 0$. 
\end{proof}

\begin{proof}[Proof (of Theorem \ref{thm:cov})]
Again let $f_0$ be distributed according to $\psi_\alpha$. Subtracting $f^\pi$  from equation \eqref{eq:td-mean}, 
\begin{equation*}
f_1 - f^\pi = (1-\alpha)(f_0-f^\pi) + \alpha\left(\T^\pi f_0-f^\pi+\xi (f_0)\right).
\end{equation*}
and taking outer products:
\begin{align*}
\left(f_1 - f^\pi\right)(f_1-f^\pi)^\tran = &(1-\alpha)^2\left(f_0-f^\pi\right)(f_0-f^\pi)^\tran \\ &+ \alpha^2\left(\T^\pi f_0-f^\pi+\xi(f_0)\right)\left(\T^\pi f_0-f^\pi+\xi(f_0)\right)^\tran \\ &+ \alpha(1-\alpha)(f_0-f^\pi)(\T^\pi f_0-f^\pi+\xi(f_0))^\tran \\ &+ \alpha(1-\alpha)(\T^\pi f_0 - f^\pi+\xi(f_0))(f_0-f^\pi)^\tran.
\end{align*}
Taking expectations on both sides, and using Lemma \ref{lem:cov-lemma}:

\begin{align*}
\bE\left[(f_1 - f^\pi)(f_1 - f^\pi)^{\tran}\right] = &(1-\alpha)^2\bE\left[(f_0-f^\pi)(f_0-f^\pi)^{\tran}\right] + \alpha^2(\gamma \P^\pi)\bE[(f_0-f^\pi)](\gamma\P^\pi)^\tran \\ &+ \alpha^2\int \C(v)\psi_a(\dd{v}) \\ &+ \alpha(1-\alpha)(\gamma \P^\pi)\bE\left[(f_0-f^\pi)(f_0-f^\pi)^\tran\right]  \\ &+ \alpha(1-\alpha)\bE\left[(f_0-f^\pi)(f_0-f^\pi)^\tran\right](\gamma \P^\pi)^\tran
\end{align*}
Since $\bE\left[(f_1 - f^\pi)(f_1 - f^\pi)^{\tran}\right] = \bE\left[(f_0 - f^\pi)(f_0 - f^\pi)^{\tran}\right]$ by stationarity, re-arranging to the LHS and factoring gives: 
\begin{align*}
(1-(1-\alpha)^2)\bE\left[(f_\alpha-f^\pi)(f_\alpha-f^\pi)^\tran\right] &= \alpha^2 (\gamma\P^\pi)\bE\left[(f_\alpha-f^\pi)(f_\alpha-f^\pi)^\tran\right](\gamma\P^\pi)^\tran \\
&\quad + \alpha(1-\alpha)(\gamma\P^\pi) \bE\left[(f_0-f^\pi)(f_0-f^\pi)^\tran\right] \\
&\quad + \alpha(1-\alpha) \bE\left[(f_\alpha-f^\pi)(f_\alpha-f^\pi)^\tran\right](\gamma\P^\pi)^\tran \\
&\quad + \alpha^2 \int \C(f)\psi_\alpha(\dd{f})
\end{align*}

For the remainder of the proof we re-write the above expression in terms of tensor products. The tensor product of two vectors $x,y$ is the matrix defined by $x \otimes y = xy^\tran$. By extension, the tensor product of two matrices $A,B$ is the operator defined by $(A \otimes B)X = AXB^\tran$. Then, the above expression can be re-written as:
\begin{align*}
(1-(1-\alpha)^2)\bE\left[(f_\alpha-f^\pi)(f_\alpha-f^\pi)^\tran\right] &= \alpha^2 (\gamma\P^\pi)^{\otimes 2}\bE\left[(f_\alpha-f^\pi)(f_\alpha-f^\pi)^\tran\right] \\
&\quad + \alpha(1-\alpha)(\gamma\P^\pi \otimes \Id) \bE\left[(f_0-f^\pi)(f_0-f^\pi)^\tran\right] \\
&\quad + \alpha(1-\alpha)(\Id \otimes \gamma\P^\pi) \bE\left[(f_\alpha-f^\pi)(f_\alpha-f^\pi)^\tran\right] \\
&\quad + \alpha^2 \int \C(f)\psi_\alpha(\dd{f}).
\end{align*}
Factoring the tensor products further gives:
$$\left[I-\left((1-\alpha)I + \alpha\gamma P^\pi\right)^{\otimes 2}\right]\bE\left[(f_\alpha - f^\pi)^{\otimes 2}\right] = \alpha^2 \int \C(f)\psi_\alpha(\dd{f})$$
We show that the matrix on the LHS is invertible. By \parencite[Corollary C.4]{puterman2014markov} it will follow from showing that $\rho\left(\left((1-\alpha)I + \alpha\gamma P^\pi\right)^{\otimes 2}\right) < 1$, where $\rho(A)$ is the spectral radius of matrix $A$. Writing $\norm{A}_\text{op} = \max_i \sum_j |A(i,j)|$ for the operator norm of a matrix $A$, and using that $\rho(A) \leq \norm{A}_\text{op}$, $\norm{A \otimes B}_\text{op} = \norm{A}_\text{op}\norm{B}_\text{op}$, and $\norm{P^\pi}_\text{op} = \norm{I}_\text{op} = 1$:
\begin{equation}\label{eq:I-A}
\norm{\left((1-\alpha)I + \alpha\gamma P^\pi\right)^{\otimes 2}}_\text{op} = \norm{(1-\alpha)I + \alpha\gamma P^\pi}^2_\text{op} \leq \left((1-\alpha) + \alpha \gamma \right)^2 < 1, 
\end{equation}
where the last inequality followed since $\gamma < 1$.
Finally, for the limit $\alpha \rightarrow 0$,  we use the following identity: if $A$ is such that $\norm{I-A} \leq 1$ then $\norm{A^{-1}} \leq \frac{1}{1-\norm{I-A}}$. We let $A = I - ((1-\alpha)I + \alpha\gamma\P^\pi)^{\otimes 2}$, by the calculation in \eqref{eq:I-A} we have $\norm{I-A} < 1$. So we calculate the operator norm of the covariance matrix:
\begin{align*}
\norm{\bE\left[(f_0-f^\pi)(f_0-f^\pi)^{\tran}\right]} &= \alpha^2\norm{\left[I-\left((1-\alpha)I + \alpha\gamma P^\pi\right)^{\otimes 2}\right]^{-1} \int \C(v)\psi_\alpha(\dd{v})} \\
&\leq \alpha^2 \norm{\left[I-\left((1-\alpha)I + \alpha\gamma P^\pi\right)^{\otimes 2}\right]^{-1}}\norm{\int \C(v)\psi_\alpha(\dd{v})} \\
&\leq \alpha^2 \frac{1}{1-\norm{I - I + \left((1-\alpha)I + \alpha\gamma P^\pi\right)^{\otimes 2}}} \norm{\int \C(v)\psi_\alpha(\dd{v})} \\
&= \alpha^2 \frac{1}{1-\norm{\left((1-\alpha)I + \alpha\gamma P^\pi\right)^{\otimes 2}}} \norm{\int \C(v)\psi_\alpha(\dd{v})} \\
&= \alpha^2 \frac{1}{1-\norm{\left((1-\alpha)I + \alpha\gamma P^\pi\right)}^2} \norm{\int \C(v)\psi_\alpha(\dd{v})} \\
&\leq \alpha^2 \frac{1}{1-(1-\alpha+\alpha\gamma)^2} \norm{\int \C(v)\psi_\alpha(\dd{v})} \\
\end{align*}

Finally, since the state space is bounded in $[0,\textsc{Rmax}/(1-\gamma)]^n$, we have $(\widehat{\T}f)_i \leq \textsc{Rmax}/(1-\gamma)$ and $(\T f)_i \leq \textsc{Rmax}/(1-\gamma)$ for each $i$. Then, we have $|\xi_\omega(f)_i \xi_\omega(f)_j|= |(\widehat{\T}f)_i(\T f)_j - (\T f)_i(\widehat{\T}f)_j - (\T f)_j(\widehat{\T}f)_j + (\T f)_j (\T f)_i| \leq 4\frac{\textsc{Rmax}^2}{(1-\gamma)^2} $  
Thus we have $\norm{\C(f)} \leq 4 \frac{\textsc{Rmax}^2}{(1-\gamma)^2} \coloneqq M$ and thus
\begin{align*}
\norm{\bE\left[(f_0-f^\pi)(f_0-f^\pi)^{\tran}\right]} &\leq M \frac{\alpha^2}{1-(1-\alpha+\alpha\gamma)^2} \stackrel{\alpha \rightarrow 0}{\longrightarrow} 0 
\end{align*}

For the concentration inequality, we will use a multivariate Chebyshev inequality \parencite[Theorem 3.1]{marshall1960multivariate}, whos statement is as follows:
\begin{theorem}\label{thm:cheby}
Let $X = (X_1,...,X_n)$ be a random vector with $\bE X = 0$ and $\bE[ X^T X] = \varSigma$. Let $T = T_+ \cup \left\{x: -x \in T_+\right\}$, where $T_+ \subseteq \bR^n$ is a closed, convex set. If $A= \left\{ a \in \bR^n : \langle a, x \rangle \geq 1 \  \forall x \in T_+ \right\}$, then 
$$
\bP\left\{ X \in T \right\} \leq \inf_{a \in A} a^\tran \varSigma a  
$$
\end{theorem}

Let $\varepsilon > 0$. We first bound $a^\tran \varSigma a$ with the operator norm of $\varSigma$. Note that 
\begin{align*}
a^\tran \varSigma a &= \sum_i a_i (\varSigma a)_i \\ 
&\leq \sum_i a_i \norm{\varSigma a} \leq n \norm{\varSigma}_\text{op}{\norm{a}}^2
\end{align*}
We define $T_+$ to be the intersection of half-planes the $\left\{x | x_i \geq \varepsilon\right\}$, so that $T_+ = \left\{ x | x_i \geq \varepsilon \ \forall i\right\}$. Since the half-planes are closed and convex, $T_+$ is also closed and convex since it is an intersection of closed and convex sets.Then, $T = T_+ \cup \left\{x: -x \in T_+\right\} = \left\{ x | x_i \geq \varepsilon \ \forall i \text{ or } x_i \leq -\varepsilon \ \forall i \right\}$. Note that $x \in T \iff \min_i \lvert x_i \rvert \geq \varepsilon$. We define $X = f_\alpha - f^\pi$ which has zero-mean. Finally, Theorem \ref{thm:cheby} states that 
$$ \bP\left\{X \in T \right\} = \bP\left\{ f_\alpha - f^\pi \in T \right\} \leq \inf_{a \in A} a^T \varSigma a \leq n \norm{\varSigma}_\text{op} \inf_{a \in A} \norm{a}^2.$$
Note that $\inf_a \norm{a}^2$ is bounded since $a = (\frac{1}{n\varepsilon},\frac{1}{n\varepsilon},....,\frac{1}{n\varepsilon})$ is in $A$ and $\norm{a}^2 = \frac{1}{(n\varepsilon)^2}$. So $n \inf_{a \in A} \norm{a}^2 \leq C$ for some constant $C$ independent of $\alpha$. From the previous result, we can take the limit of $\alpha \rightarrow 0$ of $\norm{\varSigma}_\text{op} = \norm{\bE\left[ (f_\alpha-f^\pi)(f_\alpha-f^\pi)^\tran \right]}_\text{op}$ and obtain:
$$
\bP\left\{ f_\alpha - f^\pi \in T \right\} = \bP\left\{\min_i \lvert f_\alpha (i) - f^\pi (i) \rvert \geq \varepsilon\right\}  \leq C \cdot \norm{\bE\left[ (f_\alpha-f^\pi)(f_\alpha-f^\pi)^\tran \right]}_\text{op}  \rightarrow 0 
$$

\end{proof}

\section{Proofs of Section \ref{sec:mcc}}\label{sec:mcc-proof}

\begin{lemma}\label{lemma:greedy-app}
Suppose $\pi'(s) = \argmax_{a}Q^\pi(s,a)$ for each $s$. Then $K(\pi,\pi') = \bP\{\pi' \text{ is greedy with respect to } \G^\pi \} > 0$.
\end{lemma}

We will prove an intermediate probability lemma. Let $X_1,...,X_n$ be mutually independent random variables bounded in $[a,b]$, and $F_i(x) = \bPc{X_i \leq x}$ denote the cumulative density functions of $X_i$ for $i=2,..,n$. Note that 
\begin{align*}\label{eq:max-cdf}
\bPc{X_1 \geq X_2, X_1 \geq X_3, ..., X_1 \geq X_n } &= \int_a^b \int_{a}^{x_1} \cdots \int_{a}^{x_1}\dd\bP(x_1,...,x_n) \\
&= \int_a^b \int_{a}^{x_1} \cdots \int_{a}^{x_1}\dd\bP_1(x_1)\dd\bP_2(x_2)\dd\bP_n(x_n) \quad \text{by mutual independence} \\
&= \int_a^b F_2(x_1) \cdots F_n(x_1) \dd\bP_1(x_1) \\
&= \bE\left[F_2(X_1)F_3(X_1)\cdots F_n(X_1) \right] . \numberthis
\end{align*} 
Then, we have:
\begin{lemma}\label{lemma:prob-lemma}
Suppose that $\bE[F_i(X_1)] > 0 \ \forall i=2,...,n$. Then also
$$
\bE\left[F_2(X_1) \cdots F_n(X_1) \right] >0
$$
\end{lemma}
\begin{proof}
It is easy to see that $H(x_1) = \Pi_{i = 2}^n F_i(x_1)$ is also a CDF. In particular, $H$ starts at $0$, ends at $1$, and it monotone and right-continuous. In fact, by Equation \eqref{eq:max-cdf} it corresponds to the CDF of $\max(X_2,...,X_n)$. Assume for a contradiction that $\bE\left[F_2(X_1) \cdots F_n(X_1) \right] =0$. By positivity, monotonicity, and right-continuity, we have that $H(x_1) = 0 \ \forall x_1 \in [a,b)$. Then, for every $x$ we have 
$$ 
H(x) = 0 \implies F_i(x) = 0 \text{ for some }i.
$$
Since we have $H(b)=1$ and $H(x)=0$ otherwise, note that there must exist one $i'$ such that $F_{i'}(b)=1$ and $F_{i'}(x)=0$ otherwise. If not, then for all $i$ there exists a $\varepsilon_i > 0$ such that $F_i(b-\varepsilon_i) > 0$. By monotonicity, $F_i(b-\min_i \varepsilon_i) > 0 \ \forall i$, and thus $H(b-\min_i\varepsilon_i)>0$. Thus we have $\bE[F_{i'}(x)]=0$, a contradiction.
\end{proof}

\begin{proof}[Proof (Lemma \ref{lemma:greedy-app})]
Note that
\begin{equation*}
    K(\pi,\pi') = \bPc{\pi' \text{ is greedy with respect to } \G^\pi} = \bPc{\text{for each } s, \G^\pi(s,\pi'(s)) \geq \G^\pi(s,a) \ \forall a} .
\end{equation*}
Fix a state $s$, write $X_i(s) := G^\pi(s, a_i)$, and without loss of generality assume that $\pi'(s) = a_1$. We first show that $\bE[ F_i(X_1)] > 0$, i.e. $\bPc{G^\pi(s, a_1) \ge G^\pi(s, a)} > 0$ for all $a$. Suppose that it is not so, and pick $a$ such that $\bPc{G^\pi(s, a_1) \ge G^\pi(s, a)} = 0$. Then
\begin{align*}
Q^\pi(s,a_1) &= \bE\left[\G^\pi(s,a_1)\right] \\ 
&= \bP\{\G^{\pi}(s,a_1) \geq \G^{\pi}(s,a)\}\bE\left[\G^\pi(s,a_1) \mid \{\G^{\pi}(s,a_1) \geq \G^{\pi}(s,a)\}\right] \\ &+ \bP\{\G^{\pi}(s,a_1) < \G^{\pi}(s,a)\}\bE\left[\G^\pi(s,a_1) \mid\{\G^{\pi}(s,a_1) < \G^{\pi}(s,a)\}\right] \\
&= 0 + \bE\left[\G^\pi(s,a_1) |\{\G^{\pi}(s,a_1) < \G^{\pi}(s,a)\}\right] \\
&< \bE\left[\G^\pi(s,a)\right] = Q^\pi(s,a),
\end{align*}
which contradicts the fact that $\pi'$ is greedy wrt $Q^\pi$. Hence $\bE [ F_i(X_1)] > 0$, and we apply Lemma \ref{lemma:prob-lemma} to this set to conclude that for each $s$,
\begin{equation*}
    \bPc{G^\pi(s, a_1) \ge G^\pi(s, a), \forall a} > 0 .
\end{equation*}
Because the returns are mutually independent, we further know that
\begin{equation*}
    \bPc{G^\pi(s, a_1) \ge G^\pi(s, a), \forall s, a} = \prod_{s \in \mathcal{S}} \bPc{G^\pi(s, a_1) \ge G^\pi(s, a), \forall a} > 0,
\end{equation*}
completing the proof.
\end{proof}

\section{On Wasserstein convergence vs. total variation convergence}\label{sec:tv-conv} 
Recall the definition of the Total Variation metric:
\begin{definition}
The total variation metric between probability measures is defined by:
$$
d_\texttt{TV}(\mu,\nu) = \sup_{\B \in \Borel(\bR^\d)}|\mu(\B)-\nu(\B)|,
$$
for $\mu,\nu \in \Dists(\bR^\d)$.
\end{definition}

Consider a bandit with a single arm that has a deterministic reward of 0. Consider any of the classic algorithms covered in this paper, which will sample a target of 0 at every iteration. It is easy to see that the unique stationary distribution of the algorithm in this instance is a Dirac distribution at $0$ (denoted $\delta_0$). 

Suppose a step-size of $\alpha<1$. If we initialize with some $f_0 \neq 0$ then the algorithm will never converge to the true stationary distribution in Total Variation distance. This is because a Dirac distribution at any $x\neq 0$ is always a constant distance of $1$ away from a Dirac at $0$. In other words, 
$$
d_\texttt{TV}(\delta_0,\delta_{f_n}) = 1 \quad \forall n
$$
despite the fact that $f_n \rightarrow 0$. On the other hand, we have
$$
\W(\delta_0,\delta_{f_n}) \rightarrow 0,
$$
since the Wasserstein metric takes into consideration the underlying metric structure of the space.\footnote{In particular, the Wasserstein metric isometrically embeds the original metric space into the space of probability measures \parencite{mardare2018free}.}

\end{appendices}

\end{document}